\documentclass[11pt]{article}
\usepackage[margin=1in]{geometry}
\usepackage{amsmath,amssymb,amsthm}
\usepackage{algorithm,algpseudocode}
\usepackage{booktabs}
\usepackage{hyperref}
\usepackage{listings}
\usepackage{tikz}

\newtheorem{proposition}{Proposition}

\title{Loom: A Scalable Analytical Neural Computer Architecture}
\author{Mehmet Kerem Turkcan\\
Columbia University\\
{\tt\small mkt2126@columbia.edu}
}

\date{}

\begin{document}
\maketitle

\begin{abstract}
We present Loom, a computer architecture that executes programs compiled from C
inside a looped transformer whose weights are derived analytically. The
architecture implements a 22-opcode instruction set in 8 transformer layers. Each forward
pass executes one instruction; the model is applied iteratively until the
program counter reaches zero. The full machine state resides in a single
tensor $X \in \mathbb{R}^{d \times n}$ of fixed size, and every step
has fixed cost for fixed $d$ and $n$, independent of program length or
execution history. The default
configuration uses $d = 155$ and $n = 1024$, yielding 4.7 million parameters
and 928 instruction slots. A compact configuration at $d = 146$ and $n = 512$
suffices for a 9$\times$9 Sudoku solver (284 instructions). The weights are program-independent: programs
live in the state tensor, and the same fixed-weight model executes any compiled
program. We make Loom source code publicly available at \url{https://github.com/mkturkcan/Loom}.
\end{abstract}

\section{Introduction}

Giannou et al.~\cite{giannou2023} showed that a looped transformer with
analytically constructed weights can execute the SUBLEQ instruction, a
subtract-and-branch-if-nonpositive operation that is Turing complete on its own.
Their construction uses 10 transformer layers for this single instruction.

We extend this result from 1 opcode in 10 layers to 22 opcodes in 8 layers.
The core technique is \emph{opcode-as-operand-routing}: all 22 operations are
mapped to operand preparation for a shared subtract core, so the architecture
needs only one arithmetic layer regardless of how many opcodes it supports. A
single-layer direct subtraction method replaces the classical three-layer
approach. Branch-flag computation and program counter (PC) increment are
merged into one layer.
Scratchpad correction is folded into the indirect-access layer. Together, these
changes reduce the layer count while expanding the instruction set.

A key instruction set architecture (ISA) optimization is the STORE instruction,
which provides indirect memory write and complements the existing LOAD for
indirect read. Without STORE, variable-index
array writes require an $O(m)$ dispatch tree with one branch per array element.
Adding STORE reduces the 9$\times$9 Sudoku solver from 1{,}085 to 284
instructions, enabling it to fit in the compact $146 \times 512$ configuration
instead of the larger $164 \times 2048$.

A C compiler translates a subset of C to the ISA. The compiler is implemented
in both Python and JavaScript; the JavaScript version enables in-browser
compilation and execution with no server. We demonstrate the architecture on
programs including a sorting visualizer, a playable Snake game (210
instructions), a raycasting game (531 instructions on $155 \times 1024$),
and a 9$\times$9 Sudoku solver (284 instructions on $146 \times 512$).
Beyond standalone execution, a Loom-style architecture could embed
deterministic algorithmic logic such as state machines, decision trees and constraint
checks directly into the attention layers of a perception model, guaranteeing
correct execution of safety-critical rules without relying on learned
approximations.

\paragraph{Contributions.}
\begin{enumerate}
  \item A 22-opcode ISA (21 extended opcodes plus SUBLEQ) realized in 8
    analytically weighted transformer layers.
  \item Opcode-as-operand-routing: all operations reduce to operand preparation
    for a shared subtract core, avoiding per-opcode execution layers.
  \item Single-layer direct subtraction via a 6-threshold-per-bit rectified
    linear unit (ReLU) pattern
    with built-in carry propagation.
  \item Indirect memory write (STORE) via L2 address rewriting: the
    feedforward network (FFN)
    overwrites the write-address field in the scratchpad with a resolved pointer,
    redirecting L5's write attention to the dereferenced target. This reduces
    compiled code size by up to 75\% for programs with variable-index array writes.
  \item A C compiler targeting the ISA, with both offline and in-browser
    execution paths.
  \item Scale-independent design: the same 8-layer architecture, compiler, and
    ISA work at $146 \times 512$, $155 \times 1024$, and $164 \times 2048$.
    Programs are recompiled for each configuration but the source code is
    unchanged.
  \item Field-programmable gate array (FPGA) synthesis on a Xilinx Alveo U200
    with argmax attention (no
    $n \times n$ matrix).
\end{enumerate}

\section{Related Work}\label{sec:related}

\paragraph{The programmable-computer framework.}
Giannou et al.~\cite{giannou2023} showed that a fixed-depth transformer can
execute SUBLEQ with analytically constructed weights. Subsequent work has
extended this framework in several directions.
Liang et al.~\cite{liang2025} show that a 23-layer ReLU-MLP can also serve as a
programmable computer without attention.
Back de Luca and Fountoulakis~\cite{backdeluca2024graphs} extend the original
construction with a dual-memory SUBLEQ variant for graph algorithms.
Huang et al.~\cite{huang2025hypergraph} generalize looped transformer algorithmic
reasoning to hypergraphs.
Xu and Sato~\cite{xu2025expressive} establish formal approximation rates for
looped transformers and introduce timestep encoding.
All of these remain single-instruction. Our work adds 21 extended opcodes
(22 total with SUBLEQ) and reduces the layer count.

\paragraph{Exact neural computation.}
Back de Luca et al.~\cite{backdeluca2025} prove that two-layer networks can
learn to execute binary addition, multiplication, and
Subtract-and-Branch-if-Negative exactly, using the NTK framework.
Saldyt and Kambhampati~\cite{saldyt2024} augment LLaMA3 with memory,
registers, and basic operations, compiling algorithms into differentiable
libraries.

Tzamos~\cite{tzamos2026} implements a WebAssembly interpreter inside a
7-layer autoregressive transformer with $d = 36$ and 18 two-dimensional
attention heads. Programs are supplied as input tokens and executed as a growing
trace at $O\!\left(\log t\right)$ per step via HullKVCache. Their approach is
architecturally distinct: execution logic lives in the weights, programs are
input tokens, and state grows with execution time. In Loom, the weights are
program-independent, programs live in the state tensor, and state is fixed.

\paragraph{Analytical weight construction.}
Tracr~\cite{lindner2023} compiles RASP programs~\cite{weiss2021} to
transformer weights for interpretability research.
ALTA~\cite{shaw2024} extends RASP with loops and compiles to Universal
Transformers.
Zhai et al.~\cite{zhai2025} show that trained transformers can perform
compilation.
Zhang et al.~\cite{zhang2026} extract executable programs from trained
transformer weights, working in the reverse direction.

\paragraph{Transformer Turing completeness.}
Schuurmans et al.~\cite{schuurmans2024} show that autoregressive decoding
realizes universal computation via Lag systems.
Li and Wang prove Turing completeness for constant-parameter
transformers~\cite{li2025constantbit} and construct efficient TM simulations
with sparse attention~\cite{li2025efficient}.
Jiang et al.~\cite{jiang2025softmax} settle the question for softmax attention.
These results establish theoretical universality; our work demonstrates a
concrete, executable architecture.

\paragraph{Program execution and simulation.}
Hou et al.~\cite{hou2025} construct RASP programs that simulate arbitrary
Turing machines by decomposing algorithms into TM steps.
La Malfa et al.~\cite{lamalfa2025} study LLMs simulating code execution as a
reasoning proxy.

\section{Architecture}\label{sec:arch}

\subsection{State Matrix}

The computer operates on a state matrix $X \in \mathbb{R}^{d \times n}$. The
mutable data entries (memory values, PC bits, scratchpad registers) are bipolar,
taking values in $\{-1, +1\}$. Two metadata regions use fixed non-bipolar
values: the indicator row is $+1$ in scratchpad columns and $0$ elsewhere, and
column~0's position encoding is all zeros; see Section~\ref{sec:state-layout}.
L8's error correction operates only on the persistent bipolar state (memory
and PC) and leaves these fixed entries unchanged. Each column represents a location:
columns $0$ through $s - 1$ form a scratchpad for internal routing
($s = 32$ in all reported configurations), columns $s$
through $s + m - 1$ hold data memory, and columns $s + m$ through $n - 1$ hold
instructions. The rows are partitioned into functional regions described in
Section~\ref{sec:state-layout}. One forward pass through the 8 layers executes
one instruction. The program terminates when the program counter reaches
column~0.

\subsection{Attention Mechanism}\label{sec:attention}

Each layer consists of multi-head attention followed by a feedforward network,
both with residual connections:
\begin{align}
  A &= X + \sum_{h=1}^{H} V_h X \cdot \mathrm{softmax}\!\left(\lambda \cdot X^\top K_h^\top Q_h X\right) \label{eq:attn} \\
  Y &= A + W_2 \cdot \mathrm{ReLU}\!\left(W_1 A + b_1\right) + b_2 \label{eq:ffn}
\end{align}
where $Q_h, K_h \in \mathbb{R}^{r_Q \times d}$ are the query and key
projection matrices, $V_h \in \mathbb{R}^{d \times d}$ is the value matrix,
$W_1 \in \mathbb{R}^{r_{\mathrm{FFN}} \times d}$, $W_2 \in \mathbb{R}^{d
\times r_{\mathrm{FFN}}}$, and $b_1 \in \mathbb{R}^{r_{\mathrm{FFN}} \times
n}$, $b_2 \in \mathbb{R}^{d \times n}$ are biases. All weight matrices are set
analytically; none are learned.

The softmax is applied column-wise, the standard convention: each column of the
attention matrix sums to one, so each target column independently selects which
source columns to read from.
Two aspects differ from a typical language-model transformer.
First, there is no causal mask: every column can attend to every other
column, enabling the program counter in the scratchpad to attend to any
instruction column. Second, the softmax temperature $\lambda = 10$ sharpens the
attention distribution, approximating a hard argmax that selects the single
best-matching column.

All attention heads in Loom use symmetric $Q = K$, so the score between columns
$i$ and $j$ is the inner product $\left(Q X_i\right)^\top \left(Q X_j\right)$.
This symmetry is deliberate: it allows a
column to match against another by extracting the same features from both, such
as matching a program counter against a position encoding.

Four of the eight layers have $Q_h = K_h = V_h = 0$ for all heads and perform
pure FFN computation without attention. These are L4, L6, L7, and L8.

\subsection{Bipolar Encoding}

All values use bipolar encoding: binary 0 maps to $-1$ and binary 1 maps to
$+1$. An $N$-bit signed integer $v$ with two's complement representation
$b_0 b_1 \cdots b_{N-1}$ is stored as
\[
  \hat{v} = \left(2b_0 - 1,\; \ldots,\; 2b_{N-1} - 1\right)
  \in \{-1,+1\}^N,
\]
where $b_0$ is the most significant bit. Addresses are encoded identically
using $\ell = \log_2 n$ bits.

Bipolar encoding is essential for the ReLU-based gating throughout the
architecture. A gate that must fire when bit $b_i = 1$ simply reads the
bipolar value $\hat{b}_i$ with a positive weight:
\[
  \mathrm{ReLU}\!\left(\hat{b}_i - \epsilon\right)
  \;\text{is nonzero only when}\; \hat{b}_i = +1.
\]
Multi-bit pattern matching sums the bipolar values with a threshold equal to
the pattern length minus a small offset.

\section{State Layout}\label{sec:state-layout}

The $d$ rows of the state matrix are partitioned as follows. Let $\ell =
\log_2 n$ and $N$ denote the integer precision in bits.

\begin{table}[h]
\centering
\begin{tabular}{lll}
\toprule
Region & Rows & Size \\
\midrule
Commands & $0$ to $3\ell - 1$ & $3\ell$ \\
Memory & $3\ell$ to $3\ell + N - 1$ & $N$ \\
Scratchpad & $3\ell + N$ to $6\ell + 3N - 1$ & $3\ell + 2N$ \\
Program counter & $6\ell + 3N$ to $7\ell + 3N - 1$ & $\ell$ \\
Position encoding & $7\ell + 3N$ to $8\ell + 3N - 1$ & $\ell$ \\
Buffer & $8\ell + 3N$ to $8\ell + 7N + \ell - 1$ & $4N + \ell$ \\
Address tags & $8\ell + 7N + \ell$ to $9\ell + 8N - 1$ & $N$ \\
Indicator & $9\ell + 8N$ & $1$ \\
\bottomrule
\end{tabular}
\caption{Row layout of the state matrix. $d = 9\ell + 8N + 1$.
For $\ell = 10$, $N = 8$: $d = 155$.}
\end{table}

The \textbf{commands} region stores the three operand addresses $a, b, c$ of
the current instruction, each encoded in $\ell$ bipolar bits. The
\textbf{scratchpad} contains two $N$-bit ALU operand registers,
$\mathit{scr\_sub}$ and $\mathit{scr\_min}$, followed by the decoded operand
addresses $\mathit{addr\_a}$, $\mathit{addr\_b}$, $\mathit{addr\_c}$, each of
width $\ell$. The \textbf{buffer} holds three fetched memory operands
$\mathit{buf\_a}$, $\mathit{buf\_b}$, $\mathit{buf\_c}$ of width $N$ each,
temporary storage $\mathit{find\_temp}$ of width $N$ for the FIND instruction,
and $\mathit{load\_temp}$ of width $\ell$ for the LOAD instruction. FIND and
LOAD require separate temporary buffers because both L3 attention heads fire on
every step; shared storage would cause cross-head interference. The
\textbf{address tags} store each memory column's own index in $N$ bipolar bits,
enabling content-addressable lookup by the FIND instruction. The
\textbf{indicator} row equals $+1$ in scratchpad columns $0, \ldots, s - 1$ and
$0$ elsewhere, gating FFN rows that must operate only on the scratchpad.

The \textbf{position encoding} rows store the bipolar representation of each
column's index, serving as the target for attention-based address matching.
Column~0 has its position encoding set to all zeros rather than the bipolar
encoding of zero. This prevents L1's symmetric attention from matching
column~0's position against the program counter when the PC points elsewhere,
since the all-zero vector contributes no correlation in the $Q = K$ dot
product.

\section{Pipeline}\label{sec:pipeline}

The eight transformer layers form the following pipeline.

\begin{table}[h]
\centering
\begin{tabular}{cllc}
\toprule
Layer & Stage & Function & Attn.\ heads \\
\midrule
L1 & Fetch & Read instruction at PC & 1 \\
L2 & Decode + Read & Read operands, opcode-gated routing & 3 \\
L3 & Indirect access & LOAD/FIND, scratchpad snap & 2 \\
L4 & Execute & Direct borrow-chain subtraction & 0 \\
L5 & Write & Store result to memory & 2 \\
L6 & Branch + PC$+$1 & Condition flag, PC increment & 0 \\
L7 & Branch select & Choose branch target or PC$+$1 & 0 \\
L8 & Error correction & Snap persistent state (memory, PC) to $\pm 1$ & 0 \\
\bottomrule
\end{tabular}
\caption{The 8-layer pipeline. Layers 4, 6, 7, and 8 have $Q = K = V = 0$ and
use FFN-only computation.}
\end{table}

\subsection{L1: Instruction Fetch}

L1 reads the instruction at the current program counter. Its single attention
head uses symmetric $Q = K$ that extract both the program counter and the
position encoding:
\[
  Q[i, \mathit{idx\_pc} + i] = 1, \quad Q[i, \mathit{idx\_pos} + i] = 1
\]
for $i = 0, \ldots, \ell - 1$. For the scratchpad column, $Q X_0$ yields the
PC bits plus the all-zero position encoding of column~0, giving just the PC.
For an instruction column $j$, $Q X_j$ yields the zero PC plus the position
encoding of $j$. The score $\left(Q X_0\right)^\top \left(Q X_j\right)$ is thus
$\mathrm{PC} \cdot \mathrm{pos}_j$, maximized when column $j$'s position
matches the PC.

The value matrix copies the instruction encoding from the matched column into
the buffer. The FFN transfers these buffer values into the scratchpad command
registers using sign-detection gates. Each command bit gets two ReLU rows that
detect whether the buffer value is positive or negative and write $\pm 1$ to
the corresponding scratchpad position. The buffer values arrive scaled by
approximately $0.5$ from attention normalization, so the $W_1$ weights use a
factor of $4.0$ and the $W_2$ weights use $\pm 0.5$. The FFN also clears stale
values. Total FFN size: $6 \times 3\ell$ rows.

\subsection{L2: Operand Read and Opcode Decode}\label{sec:l2}

Three attention heads read $\mathrm{col}[a]$, $\mathrm{col}[b]$, and
$\mathrm{col}[c]$ into $\mathit{buf\_a}$, $\mathit{buf\_b}$, and
$\mathit{buf\_c}$ respectively. Each head uses symmetric $Q = K$ matching the
corresponding operand address against position encodings. The value matrix of
Head~1 additionally copies the memory value into $\mathit{scr\_sub}$, and the
value matrix of Head~2 additionally copies into $\mathit{scr\_min}$, pre-loading
default operand values that the FFN then corrects per opcode.

The FFN performs opcode-gated routing. Two execution modes exist:

\paragraph{SUBLEQ mode.} When $a \geq s$, i.e.\ $a \geq 32$ in the default
configuration, the instruction is a classical SUBLEQ. Mode detection examines
the most significant bit of the $a$ address that distinguishes the scratchpad
range from memory. In bipolar encoding, SUBLEQ mode fires when
$\mathit{addr\_a}[\ell - 6] = +1$, i.e.\ the bit representing value $32$ is
set. Each gating row reads this single bit with weight $1.0$ and bias $-0.5$,
so the ReLU fires only when the bit is $+1$. The FFN copies $\mathit{buf\_a}$
into $\mathit{scr\_sub}$ using $4N$ rows gated on this condition. A separate
block of $4N$ rows copies $\mathit{buf\_b}$ into $\mathit{scr\_min}$; this
copy is ungated and runs in both modes.

\paragraph{Extended mode.} When $a < s$, i.e.\ $a < 32$, the value of $a$
encodes one of 21 extended opcodes (HALT through STORE, numbered 0 to 20).
Together with SUBLEQ, this gives 22 distinct operations. Each opcode $k$ has
dedicated FFN rows that fire only when the $\ell$-bit bipolar address pattern
matches. A gating row for opcode $k$ sets its $W_1$ weights to the bipolar
encoding of $k$ on the address dimensions with bias $-\ell + 0.5$, causing the
ReLU to fire when all $\ell$ bits match.

After L2, the scratchpad contains two $N$-bit operands whose subtraction in L4
produces the correct result for any of the 22 opcodes. The FFN allocates $80N$
rows. For $N = 8$, this is 640 rows.

\subsection{L3: Indirect Access and Scratchpad Correction}

Two attention heads implement indirect memory access.

\paragraph{Head 1: LOAD.} L2 copies the pointer value $\mathrm{col}[c]$
into $\mathit{load\_temp}$, converting the $N$-bit memory value to an $\ell$-bit
column address. Head~1 uses symmetric $Q = K$ on $\mathit{load\_temp}$ and
$\mathit{pos\_enc}$, matching the pointer against column positions. The value
matrix copies the matched column's memory value into $\mathit{scr\_min}$ with
weight $2.0$, overriding the default value placed there by L2.

\paragraph{Head 2: FIND.} L2 copies the search value $\mathrm{col}[c]$ into
$\mathit{find\_temp}$. Head~2 uses symmetric $Q = K$ on $\mathit{find\_temp}$
and the memory rows, matching against memory contents. The value matrix copies
the matched column's address tag into $\mathit{scr\_min}$, yielding the index
of the matching entry. The compiler assumes FIND targets arrays with unique
values at runtime; it does not verify this property. Duplicate matches produce
a weighted mixture of tags with no defined semantics; see
Section~\ref{sec:opcodes}.

\paragraph{Scratchpad snap.} The FFN clears temporary buffers and snaps
$\mathit{scr\_sub}$ and $\mathit{scr\_min}$ to exact bipolar values using six
ReLU thresholds per bit. The 3-head attention routing in L2 produces outputs
that deviate further from $\pm 1$ than the single-instruction case, making this
correction necessary for numerical stability before the ALU stage.

\subsection{L4: Direct Subtraction}\label{sec:l4}

L4 has no attention. It computes $\mathit{scr\_min} \leftarrow \mathit{scr\_min}
- \mathit{scr\_sub}$ using a single-layer borrow-chain subtraction. This is the
sole arithmetic layer; all 22 opcodes share it.

For each output bit position $i$ ($0 \leq i \leq N-1$, most significant bit
(MSB) first), define the
weighted partial difference
\[
  D_i \;=\; \sum_{j=i}^{N-1} \bigl(\hat{a}_j - \hat{b}_j\bigr) \cdot 2^{N-2-j}
\]
where $\hat{a}_j, \hat{b}_j \in \{-1,+1\}$ are the bipolar values of
$\mathit{scr\_min}[j]$ and $\mathit{scr\_sub}[j]$. The weight $2^{N-2-j}$
is half the positional value of bit $j$; the halving absorbs the bipolar
scale factor so that the sum ranges over integers when inputs are exact.

Six ReLU units, arranged in three pairs with alternating input sign, extract
the result bit. Let $P = 2^{N-1-i}$. The six pre-activations are:
\begin{align*}
  z_1 &= D_i + P + 1,    & z_2 &= D_i + P, \\
  z_3 &= -D_i,           & z_4 &= -D_i - 1, \\
  z_5 &= D_i - P + 1,    & z_6 &= D_i - P.
\end{align*}
The result bit is
\[
  r_i = 2\bigl[\mathrm{ReLU}(z_1) - \mathrm{ReLU}(z_2) +
  \mathrm{ReLU}(z_3) - \mathrm{ReLU}(z_4) +
  \mathrm{ReLU}(z_5) - \mathrm{ReLU}(z_6)\bigr] - 3.
\]

\begin{proposition}\label{prop:l4}
For all $N$-bit two's complement inputs, $r_i \in \{-1,+1\}$ and equals the
bipolar encoding of bit $i$ of $a - b$.
\end{proposition}

\textit{Proof.} For exact bipolar inputs, $D_i$ is an integer in
$[-(2P-1),\; 2P-1]$. We use the identity
$\mathrm{ReLU}(t+1) - \mathrm{ReLU}(t) = \mathbf{1}[t \geq 0]$
for integer $t$. Applying it to the three ReLU pairs:
\begin{align*}
  \mathrm{ReLU}(z_1) - \mathrm{ReLU}(z_2)
    &= \mathbf{1}[D_i \geq -P], \\
  \mathrm{ReLU}(z_3) - \mathrm{ReLU}(z_4)
    &= \mathbf{1}[D_i \leq -1], \\
  \mathrm{ReLU}(z_5) - \mathrm{ReLU}(z_6)
    &= \mathbf{1}[D_i \geq P].
\end{align*}
Therefore
\[
  r_i = 2\bigl(\mathbf{1}[D_i \geq -P]
  + \mathbf{1}[D_i \leq -1]
  + \mathbf{1}[D_i \geq P]\bigr) - 3,
\]
which yields four intervals over the range of $D_i$:
\begin{center}
\begin{tabular}{ll}
$D_i \in [-(2P-1),\, -P-1]$: & $r_i = 2(0+1+0)-3 = -1$, \\
$D_i \in [-P,\, -1]$:        & $r_i = 2(1+1+0)-3 = +1$, \\
$D_i \in [0,\, P-1]$:        & $r_i = 2(1+0+0)-3 = -1$, \\
$D_i \in [P,\, 2P-1]$:       & $r_i = 2(1+0+1)-3 = +1$.
\end{tabular}
\end{center}
These four intervals partition the full range of $D_i$, and the resulting
$r_i \in \{-1,+1\}$ matches the bipolar encoding of bit $i$ of $a - b$
in two's complement. As a sanity check: when $a = b$, $D_i = 0$, which
falls in the third interval, giving $r_i = -1$, the bipolar encoding of
bit~0. This is correct for $a - b = 0$. Correctness has also been verified
exhaustively for $N = 8$ across all $2^{16}$ input pairs.

The classical approach requires three layers: one to flip the subtrahend bits,
one to increment for two's complement, and one to add. The 6-threshold pattern
folds all three operations into a single layer by observing that the
borrow-chain structure of subtraction maps directly to a piecewise-linear
function of the partial bit sums, which ReLU can represent exactly with six
thresholds per bit.

Total FFN: $6N$ rows for the borrow chain plus $4N$ rows for clearing the
operand registers, giving $10N$ rows.

\subsection{L5: Memory Write}

L5 uses two attention heads. Head~1 matches the $b$-operand address against
position encodings and writes $\mathit{scr\_min}$ to the matched memory column.
This head handles all opcodes.

Head~2 matches the $c$-operand address and writes $\mathit{buf\_c}$ to the
matched column. For non-SWAP opcodes, $\mathit{buf\_c}$ is cleared at memory
columns by indicator-gated rows in L2, so the Head~2 write reduces to an
identity. For SWAP, L2 loads $\mathit{buf\_b}$ into $\mathit{buf\_c}$, so
Head~2 writes $\mathrm{col}[b]$ to position $c$ while Head~1 simultaneously
writes $\mathrm{col}[c]$ to position $b$.

FFN: $10N$ rows.

\subsection{L6--L7: Conditional Branching}

Both layers have no attention.

\paragraph{L6.} Two independent FFN computations share a single layer.
The first computes a branch flag from the subtraction result and the opcode.
The second computes $\mathrm{PC} + 1$ using the same 6-threshold pattern as L4.
These computations read disjoint state and can share a layer.

The branch flag accumulates contributions from multiple row groups:
the sign bit of the subtraction result, a zero check over all result bits, a
SUBLEQ mode gate, unconditional jump rows for JMP and HALT, conditional
jump rows for JZ and JNZ, a CMP zero-correction row, and per-opcode suppress
rows for non-branching opcodes.

\paragraph{L7.} Selects between the branch target $c$ and
$\mathrm{PC} + 1$ based on the sign of the flag. Positive flag selects $c$;
non-positive selects $\mathrm{PC} + 1$.

\subsection{L8: Error Correction}

L8 has no attention. The FFN applies error correction to the persistent
bipolar state (memory and PC), clamping values back to $\{-1, +1\}$
using a 6-threshold snap pattern with deadzone $\epsilon = 0.1$. Any value
within $[\pm 1 - \epsilon, \pm 1 + \epsilon]$ is snapped to the nearest bipolar
value. The indicator row and column~0's zero position encoding are not modified
by L8, since their rows are excluded from the correction: PC rows are gated
by the indicator, and memory rows operate only at the correct columns.
This corrects accumulated floating-point drift from attention normalization and
FFN rounding across the preceding seven layers.

\section{Layer Reduction from the Baseline}\label{sec:reduction}

Giannou et al.~\cite{giannou2023} use 10 layers for SUBLEQ: 1~for fetch,
2~for operand read, 3~for subtraction, 1~for write, and 3~for branching.
Three techniques reduce this to 8 layers while expanding to 22 opcodes:

\begin{enumerate}
  \item \textbf{Direct borrow-chain subtraction.} The three-layer pipeline
    of bit-flip, increment, and addition is replaced by a single-layer
    6-threshold pattern. Three ReLU pairs per bit, with alternating input
    sign and thresholds at $\pm P$ and $0/-1$ (Section~\ref{sec:l4}),
    jointly implement complement and carry propagation. This saves 2 layers.
  \item \textbf{Merged branch and PC increment.} The branch condition and
    $\mathrm{PC} + 1$ computation read disjoint state and share one FFN.
    This saves 1 layer.
  \item \textbf{Folded scratchpad correction.} The snap-to-bipolar correction,
    originally a separate layer, is merged into L3's FFN after LOAD/FIND
    attention. This saves 1 layer.
\end{enumerate}

\section{Instruction Set}\label{sec:opcodes}

Each instruction occupies one column and consists of three $\ell$-bit addresses
$a, b, c$. When $a \geq s$, the instruction is SUBLEQ. When $a < s$, the value
of $a$ selects one of 21 extended opcodes (numbered 0 through 20). Together
with SUBLEQ, this gives 22 distinct operations. In the default configuration
with $s = 32$, SUBLEQ mode is detected by a single bipolar bit:
$\mathit{addr\_a}[\ell - 6] = +1$, the bit representing value 32.

\paragraph{Address convention.} Two address spaces coexist. Instruction
operands $a, b, c$ and the PC are \emph{absolute column indices} in
$\{0, \ldots, n-1\}$. Memory slots are identified by \emph{logical indices}
$x \in \{0, \ldots, m-1\}$, occupying column $s + x$. We write
\[
  \mathrm{col}[u] := \text{$N$-bit contents of absolute column } u,
  \qquad
  M[x] := \mathrm{col}[s + x].
\]
Instruction operands address columns directly: $\mathrm{col}[b]$,
$\mathrm{col}[c]$. LOAD, STORE, and FIND consume or produce logical memory
indices; the $+s$ translation to absolute column addresses is performed by
L2's FFN (for LOAD's $\mathit{load\_temp}$ and STORE's $\mathit{addr\_b}$
rewrite). The compiler emits absolute addresses for a given $(s, m, n)$.
Indirect pointers consumed by LOAD, STORE, and FIND are interpreted as raw
unsigned $N$-bit memory indices, even though arithmetic values use signed
two's-complement semantics. With $N = 8$, this permits $m$ up to $255$.

\subsection{SUBLEQ: Subtract and Branch if Nonpositive}

\textbf{Semantics.}
$\mathrm{col}[b] \leftarrow \mathrm{col}[b] - \mathrm{col}[a]$;
if $\mathrm{col}[b] \leq 0$ then $\mathrm{PC} \leftarrow c$, else
$\mathrm{PC} \leftarrow \mathrm{PC} + 1$.

\textbf{L2 routing.}
$\mathit{scr\_sub} = \mathit{buf\_a}$, $\mathit{scr\_min} = \mathit{buf\_b}$.
Uses $4N$ gated FFN rows. The L6 branch flag naturally produces the SUBLEQ
branch condition.

\subsection{HALT}

\textbf{Semantics.} $\mathrm{PC} \leftarrow 0$.

\textbf{L2 routing.} $\mathit{scr\_sub} = 0$; 1 row. L6 sets an unconditional
branch flag, and L7 selects $c = 0$.

\subsection{MOV: Move}

\textbf{Semantics.} $\mathrm{col}[b] \leftarrow \mathrm{col}[c]$.

\textbf{L2 routing.} $\mathit{scr\_sub} = 0$,
$\mathit{scr\_min} = \mathit{buf\_c}$. Corrects the default $\mathit{buf\_b}$
to $\mathit{buf\_c}$ using $2N$ per-bit correction rows plus 1 row to zero the
subtrahend. Total: $2N + 1$ rows.

\subsection{INC and DEC}

\textbf{Semantics.} $\mathrm{col}[b] \leftarrow \mathrm{col}[b] \pm 1$.

\textbf{L2 routing.} $\mathit{scr\_min}$ retains $\mathit{buf\_b}$.
$\mathit{scr\_sub}$ is set to the bipolar encoding of $-1$ for INC or $+1$ for
DEC. L4 computes $\mathrm{col}[b] - (-1) = \mathrm{col}[b] + 1$ or
$\mathrm{col}[b] - 1$. Each: 1
FFN row.

\subsection{ADD}

\textbf{Semantics.} $\mathrm{col}[b] \leftarrow \mathrm{col}[b] + \mathrm{col}[c]$.

\textbf{L2 routing.} $\mathit{scr\_min} = \mathit{buf\_b}$,
$\mathit{scr\_sub} = -\mathit{buf\_c}$. L4 computes
$\mathrm{col}[b] - (-\mathrm{col}[c]) = \mathrm{col}[b] + \mathrm{col}[c]$. The two's complement negation of
$\mathit{buf\_c}$ requires $2N$ rows for the one's complement and $2N$ rows
for the carry chain. Total: $4N$ rows.

\subsection{SUB}

\textbf{Semantics.} $\mathrm{col}[b] \leftarrow \mathrm{col}[b] - \mathrm{col}[c]$.

\textbf{L2 routing.} $\mathit{scr\_sub} = \mathit{buf\_c}$; $2N$ rows.

\subsection{SHL and SHR}

\textbf{Semantics.} $\mathrm{col}[b] \leftarrow \mathrm{col}[b] \ll 1$ and
$\mathrm{col}[b] \leftarrow \mathrm{col}[b] \gg 1$ respectively. SHR is arithmetic and preserves
the sign bit.

\textbf{L2 routing.} $\mathit{scr\_sub} = 0$. The minuend is corrected from
$\mathit{buf\_b}$ to the shifted value by per-bit delta rows.
SHL: $2N$ rows. SHR: $2N - 1$ rows.

\subsection{AND, OR, XOR}

\textbf{Semantics.} $\mathrm{col}[b] \leftarrow \mathrm{col}[b] \odot \mathrm{col}[c]$.

\textbf{L2 routing.} All three set $\mathit{scr\_sub} = 0$ and correct
$\mathit{scr\_min}$ from $\mathit{buf\_b}$ to the bitwise result. L4 with
$\mathit{scr\_sub} = 0$ passes $\mathit{scr\_min}$ through unchanged.

In bipolar encoding, $a \wedge b = +1$ iff both are $+1$, so AND corrects when
$\mathit{buf\_b} = +1$ and $\mathit{buf\_c} = -1$, requiring $N$ rows. OR
corrects when $\mathit{buf\_b} = -1$ and $\mathit{buf\_c} = +1$, also $N$ rows.
XOR requires corrections in both directions: $2N$ rows. Each opcode adds 1 row
to zero $\mathit{scr\_sub}$.

\subsection{JMP, JZ, JNZ, CMP}

\textbf{Semantics.}
JMP: $\mathrm{PC} \leftarrow c$.
JZ: if $\mathrm{col}[b] = 0$ then $\mathrm{PC} \leftarrow c$.
JNZ: if $\mathrm{col}[b] \neq 0$ then $\mathrm{PC} \leftarrow c$.
CMP: if $\mathrm{col}[b] < 0$ then $\mathrm{PC} \leftarrow c$.

\textbf{L2 routing.} All set $\mathit{scr\_sub} = 0$; 1 row each. Branch logic
is entirely in L6. JNZ uses a weight of $10.0$ on the opcode gate to dominate
the zero-check contribution. CMP adds a zero-correction row that converts
branch-if-nonpositive to branch-if-negative.

\subsection{LOAD: Indirect Read}

\textbf{Semantics.} $\mathrm{col}[b] \leftarrow M[\mathrm{col}[c]]$.

\textbf{Precondition.} $\mathrm{col}[c]$ must be a valid memory index in
$\{0, \ldots, m-1\}$. Out-of-range pointers cause the attention to match an
unintended column, producing undefined results.

\textbf{L2 routing.} $\mathit{scr\_sub} = 0$; the FFN converts the pointer
value $\mathit{buf\_c}$ to the $\ell$-bit position encoding of the absolute
column $s + \mathrm{col}[c]$ and writes it into $\mathit{load\_temp}$. L3 Head~1 then
matches $\mathit{load\_temp}$ against column position encodings and copies the
dereferenced value into $\mathit{scr\_min}$.

\subsection{FIND: Content-Addressable Search}

\textbf{Semantics.} $\mathrm{col}[b] \leftarrow i$ where
$M[i] = \mathrm{col}[c]$ and $0 \leq i < m$.

\textbf{Precondition.} The search value $\mathrm{col}[c]$ must occur exactly
once among memory slots $M[0]$ through $M[m-1]$. If no exact match exists, the
attention selects the nearest bipolar match or a weighted mixture of near
matches, producing a tag with no defined semantics. If multiple exact matches
exist, the result is the weighted average of their address tags, not a valid
index. The current compiler assumes FIND targets arrays whose values are
unique at runtime; it does not verify this property.

\textbf{L2 routing.} $\mathit{scr\_sub} = 0$; $\mathit{buf\_c}$ is copied into
$\mathit{find\_temp}$. L3 Head~2 matches $\mathit{find\_temp}$ against memory
contents and copies the matched column's address tag into $\mathit{scr\_min}$.

\subsection{SWAP}

\textbf{Semantics.} Atomically exchanges $\mathrm{col}[b]$ and $\mathrm{col}[c]$.

\textbf{L2 routing.} Three operations: zero $\mathit{scr\_sub}$, correct
$\mathit{scr\_min}$ from $\mathit{buf\_b}$ to $\mathit{buf\_c}$, and copy
$\mathit{buf\_b}$ into $\mathit{buf\_c}$ for the L5 second write head. Total:
$4N + 1$ rows. L5 Head~1 writes $\mathrm{col}[c]$ to position $b$; Head~2
simultaneously writes $\mathrm{col}[b]$ to position $c$.

\subsection{CMOV: Conditional Move}

\textbf{Semantics.} If $\mathrm{col}[b] < 0$ then $\mathrm{col}[b] \leftarrow \mathrm{col}[c]$.

\textbf{L2 routing.} $\mathit{scr\_sub} = 0$. Correction from $\mathit{buf\_b}$
to $\mathit{buf\_c}$ is gated on $\mathit{buf\_b}[0] = +1$, which indicates a
negative value since bit~0 is the MSB. The MSB serves as both the sign gate
and a destination bit, so the $W_1$ weight for bit~0 is $4.0$ instead of $2.0$.
Total: $2N$ rows.

\subsection{MULACC: Multiply-Accumulate Step}\label{sec:mulacc}

\textbf{Semantics.}
$\mathrm{col}[b] \leftarrow \left(\mathrm{col}[b] \ll 1\right) + \begin{cases} \mathrm{col}[c] & \text{if } \mathrm{MSB}(\mathrm{col}[b]) = 1 \\ 0 & \text{otherwise} \end{cases}$

One step of the shift-and-add multiplication algorithm. Applying MULACC $N$
times with $\mathrm{col}[b]$ initialized to the multiplier and $\mathrm{col}[c]$
holding the multiplicand computes the product for non-negative multipliers.

\textbf{L2 routing.} Combines the SHL and ADD patterns with conditional gating
on the sign bit $\mathit{buf\_b}[0]$: 1 row to clear $\mathit{scr\_sub}$,
$2N - 1$ rows for the SHL correction, and $3N$ rows for the conditional
negation of $\mathit{buf\_c}$ into $\mathit{scr\_sub}$. Total: $5N$ rows.

For signed operands, the compiler emits a wrapper that computes
$\mathrm{XOR}\!\left(a, b\right)$ for the sign, takes absolute values,
multiplies, and conditionally negates. The $\mathrm{mul}$ builtin emits 22
instructions and executes in 24 steps for $N = 8$.

\subsection{STORE: Indirect Memory Write}\label{sec:store}

$\mathrm{STORE}(b, c)$: $M[\mathrm{col}[c]] \leftarrow \mathrm{col}[b]$.
The instruction complements LOAD, which provides indirect read, with an
indirect write.

\textbf{Precondition.} $\mathrm{col}[c]$ must be a valid memory index in
$\{0, \ldots, m-1\}$. Out-of-range pointers redirect L5's write attention to
an unintended column.

The implementation reuses L5's existing write-attention mechanism. The key
insight is that L5 Head~1 writes to the column addressed by
$\mathit{addr\_b}$ in the scratchpad. By overwriting $\mathit{addr\_b}$ with
the position encoding of the dereferenced pointer address during L2's FFN, L5
writes to $M[\mathrm{col}[c]]$ instead of $\mathrm{col}[b]$.

The L2 FFN rows for STORE:
\begin{enumerate}
  \item Set $\mathit{scr\_sub} = 0$, ensuring a pass-through in L4: 1 row.
  \item Clear old $\mathit{addr\_b}$: detect and subtract each bit's current
    value, gated by STORE opcode and indicator. $2\ell$ rows.
  \item Write new $\mathit{addr\_b}$: convert the raw unsigned pointer
    $\mathrm{col}[c]$ from $\mathit{buf\_c}$ to the $\ell$-bit position
    encoding of the absolute memory column $s + \mathrm{col}[c]$. The
    implementation uses a fixed-width constant-add network whose exact row
    count depends on the bit overlap between the $N$-bit pointer and the
    $\ell$-bit column address for a given $s$.
\end{enumerate}

The always-copy mechanism sets $\mathit{scr\_min} = \mathrm{col}[b]$, the
value to write, so no additional rows are needed for the source operand. L3
is not involved; the pointer resolution happens entirely via L2's address
rewriting. L4 passes through: $\mathit{scr\_min} - 0 = \mathit{scr\_min}$.
L5 writes the value to the dereferenced address.

Without STORE, the compiler must generate an $O(m)$ dispatch tree for each
variable-index array write, one branch per array element. For an 81-element
array (as in Sudoku), each write costs $\sim$400 instructions. Adding STORE
replaces this with 4 instructions (compute pointer, STORE), reducing the
9$\times$9 Sudoku solver from 1{,}085 to 284 instructions.

\section{L2 FFN Row Budget}\label{sec:budget}

\begin{table}[h]
\centering
\begin{tabular}{lr}
\toprule
Group & Rows \\
\midrule
SUBLEQ gated copy & $4N$ \\
$\mathit{buf\_b} \to \mathit{scr\_min}$ & $4N$ \\
$\mathit{buf\_c}$ clear & $2N$ \\
INC, DEC, JZ, JNZ, JMP, HALT, CMP & $7$ \\
MOV & $2N + 1$ \\
ADD & $4N$ \\
SUB & $2N$ \\
SHL, SHR & $2N$, $2N - 1$ \\
AND, OR & $N + 1$ each \\
XOR & $2N + 1$ \\
LOAD, FIND & varies \\
SWAP & $4N + 1$ \\
CMOV & $2N$ \\
MULACC & $5N$ \\
STORE & varies \\
\midrule
Total & $\sim 340$ of $640$ \\
\bottomrule
\end{tabular}
\caption{L2 FFN row usage for $N = 8$, $\ell = 10$.}
\end{table}

\section{Model Properties}

\begin{proposition}
The model dimension is $d = 9\ell + 8N + 1$. For $n = 1024, N = 8$: $d = 155$
with approximately $4.7 \times 10^6$ parameters.
\end{proposition}

The weight matrices are extremely sparse. In the $155 \times 1024$
configuration, 99.9\% of all weight entries are zero, and the nonzero entries
take only 27 distinct values. The sparsity and discrete value structure are
direct consequences of the analytical construction: every nonzero weight
corresponds to a specific routing, gating, or threshold operation.

\subsection{Scaling}

The ISA is independent of the substrate dimensions. The same 8-layer design
and compiler work across three configurations; binaries are recompiled for
each $(s, m, n)$:
\begin{itemize}
  \item $146 \times 512$ ($\ell = 9$, $m = 160$, 320 instruction slots):
    compact configuration for programs using STORE.
  \item $155 \times 1024$ ($\ell = 10$, $m = 64$, 928 instruction slots):
    default for most demos.
  \item $164 \times 2048$ ($\ell = 11$, $m = 224$, 1{,}792 instruction slots):
    large-memory programs.
\end{itemize}
The dimension increase reflects the larger $\ell$ needed for wider position
encodings. Addresses in compiled programs are absolute column indices (e.g.\
memory address 0 is column $s$, instruction 0 is column $s + m$). The compiler
emits addresses relative to a given $s$ and $m$. When these parameters change
across configurations, the program must be recompiled; but the source code, the
compiler, and the 8-layer weight construction are identical across all three
configurations. In this sense, the \emph{architecture} is portable; the
compiled \emph{binaries} are not.

\section{C Compiler}\label{sec:compiler}

A compiler translates a subset of C to the extended ISA. The supported language
includes integer variables with $N$-bit signed arithmetic, comparisons, logical
operators, \texttt{if}/\texttt{else}, \texttt{while}, \texttt{for}, arrays with
constant size, and single-level function inlining. Builtin functions
$\mathrm{abs}$, $\mathrm{min}$, $\mathrm{max}$, $\mathrm{mul}$, and
$\mathrm{swap}$ generate inline instruction sequences.

The front end consists of a lexer and a recursive-descent parser producing an
AST. The code generator emits instruction triples $\left(\mathit{op}, b,
c\right)$ and resolves labels in a final pass. The compiler performs constant
folding and comparison strength reduction.

For $n = 1024$ with $s = 32$ and $m = 64$, the instruction budget is
$n - s - m = 928$ slots and the data budget is $m = 64$ slots. For
$n = 512$ with $m = 160$, the budget is 320 instruction slots. Variable-index
array writes compile to a single STORE instruction; without STORE, they
require an $O(m)$ dispatch tree of branches.

\section{Deployment}

The model is exported to Open Neural Network Exchange (ONNX) format.
Each forward pass executes one
instruction, and the model is applied in a loop until $\mathrm{PC} = 0$. The
ONNX model runs entirely client-side via ONNX Runtime Web, using WebGPU when
available and falling back to WebAssembly. An optimized export applies
onnxsim\footnote{\url{https://github.com/daquexian/onnx-simplifier}} graph simplification and precomputes $K^\top Q$ per
attention head, fusing two matrix multiplications into one. The $146 \times 512$
optimized model is 7.4 MB; the $155 \times 1024$ model is approximately 16 MB;
the $164 \times 2048$ model is approximately 29 MB.

An ISA interpreter that executes the compiled instructions directly, bypassing
the transformer, serves as a verification tool: both execution paths produce
identical results on all tested programs. The interpreter also provides a fast
execution mode for interactive applications. A sparse argmax engine in
JavaScript exploits the 99.9\% weight sparsity and argmax attention to execute
the transformer step using only the ${\sim}8{,}000$ nonzero weight entries,
avoiding dense matrix operations. This engine passes the same 111-test
validation suite as the ONNX path.

\subsection{FPGA Implementation}

The transformer has been synthesized and executed on a Xilinx Alveo U200 FPGA
with 6{,}840 DSP48E2 slices and 44~MB of on-chip block RAM (BRAM) and
ultra RAM (URAM). The FPGA kernel
uses the $155 \times 1024$ configuration with two key optimizations:

\paragraph{Argmax attention.} The softmax attention is replaced with a
streaming top-2 argmax per key row. For each key column $i$, the kernel scans
all query columns $j$ and tracks the two highest-scoring entries. If the gap
between them is less than 1.0 (a tie), the V contribution is split 50/50;
otherwise the winner receives full weight. This eliminates the $n \times n$
attention matrix entirely: the $155 \times 1024$ design uses 4~MB of on-chip
memory instead of the 16+~MB required by softmax. The argmax produces
bit-identical results to softmax on all tested programs because the $\lambda = 10$
temperature makes the softmax effectively one-hot except at the designed 50/50
tie points. With argmax, the kernel requires no \texttt{hls::expf} hardware,
no normalization pass, and no $n \times n$ storage.

\paragraph{Halt detection.} The kernel checks the program counter between
transformer steps and exits early when $\mathrm{PC} = 0$, avoiding wasted
computation on halted programs.

\paragraph{Resource usage.} The synthesized design uses approximately 160 of
960 URAM blocks (state, weights, and compute buffers), 224 of 2{,}160 BRAM
blocks, 38K LUTs, and 24K flip-flops. The design meets timing at 300~MHz. No
DSP slices are used in the current floating-point implementation; all
arithmetic is mapped to LUT logic.

\paragraph{Results.} The INC test ($\mathrm{mem}[0] = 5 \to 6$,
$\mathrm{PC} = 192 \to 193$) passes on real hardware. The kernel processes
columns sequentially, yielding approximately 0.3~seconds per transformer step.
For comparison, the same model runs at $\sim$10~ms/step on an NVIDIA RTX 4080
via ONNX Runtime WebGPU. The sequential column processing and all-LUT floating
point explain the 30$\times$ gap. The kernel is program-independent: weights
encode the universal ISA, and different programs are loaded via DDR without
FPGA recompilation.

\section{Demonstrations}

We evaluate the architecture on four programs of increasing complexity.

\paragraph{Sorting.} Bubble sort on an 8-element array. Each comparison and
swap is a sequence of transformer forward passes. The program compiles to
approximately 50 instructions.

\paragraph{Snake.} A playable Snake game where movement, collision detection,
and food spawning are computed by the transformer. The program compiles to 210
instructions and executes approximately 84 steps per game tick on the
$155 \times 1024$ model.

\paragraph{Raycasting.} A Wolfenstein-style raycasting game with wall rendering,
sprite enemies, and combat mechanics. The program compiles to 531 instructions
and executes approximately 500 steps per tick on the $155 \times 1024$ model.

\paragraph{Sudoku.} A 9$\times$9 Sudoku solver using iterative backtracking.
With the STORE opcode, the program compiles to 284 instructions on the
$146 \times 512$ model with 160 memory slots. Without STORE, the same program
requires 1{,}085 instructions (75\% of which are dispatch trees for
variable-index array writes) and the $164 \times 2048$ model.

All programs are written in C, compiled to the ISA, and executed as iterated
matrix multiplications through the fixed-weight transformer.

\subsection{Validation}

Correctness is tested at three levels:

\paragraph{Opcode tests (42 tests).} We ran 42 opcode-level unit tests covering
all 22 opcodes. Arithmetic opcodes are tested on positive, negative, zero,
overflow, and boundary cases; branch opcodes on taken and not-taken paths; and
indirect-memory opcodes on valid-pointer cases. Multi-step SUBLEQ countdown
loops are included. All 42 tests pass.

\paragraph{Cross-head writeback tests (19 tests).} We ran 19 tests covering
SWAP and other multi-write or cross-head interactions, including adjacent and
non-adjacent operands, self-swap, double-swap identity, and combined
operations. All 19 tests pass.

\paragraph{Compiled C tests (50 tests).} We ran 50 end-to-end compiled-program
tests including Fibonacci, GCD by subtraction, array minimum, nested
conditionals, variable-index LOAD/STORE round-trips, and multi-iteration loops
(up to 1000 ISA steps). All 50 tests pass.

An ISA interpreter that executes compiled instructions directly, bypassing the
transformer, serves as an independent verification oracle: both execution paths
produce identical results on all tested programs. The interpreter also provides
a fast execution mode for interactive demos.

\section{Conclusion}

Loom executes a 22-opcode ISA inside an 8-layer transformer with analytically
derived weights. The central insight is that all operations can be mapped to
operand preparation for a shared subtract core, so the architecture needs only
one arithmetic layer regardless of the opcode count. Direct borrow-chain
subtraction reduces this arithmetic stage to a single layer. The STORE
instruction demonstrates that ISA design has cascading effects on the physical
realization: one opcode reduced compiled code by 75\%, shrunk the transformer
from $164 \times 2048$ to $146 \times 512$, and reduced the ONNX model from
29 MB to 7.4 MB. The architecture has been verified on three software paths and one hardware
platform: an ONNX model in the browser via WebGPU or WebAssembly, a sparse JavaScript
argmax engine exploiting 99.9\% weight sparsity, and an ISA interpreter
that bypasses the transformer for fast verification. All three pass the
same 111-test validation suite. Across that suite, the top-2 argmax rule
matched the softmax path exactly. The
result is a fixed-size programmable computer that runs C programs as iterated
matrix multiplications, from browser to silicon. More broadly, embedding
analytically constructed computational layers within learned models could
provide provably correct algorithmic reasoning for safety-critical
applications such as real-time situational awareness in urban environments,
where deterministic enforcement of traffic rules, pedestrian safety zones,
and anomaly detection thresholds must not depend on learned approximations.

\section*{Acknowledgements}
This work was supported by the NSF Engineering Research
Center for Smart Streetscapes under Award EEC-2133516. FPGA experiments were conducted on the NSF PAWR COSMOS wireless testbed in West Harlem, NYC.

\bibliographystyle{plain}
\bibliography{refs}

\end{document}